\documentclass[conference]{IEEEtran}
\IEEEoverridecommandlockouts
\usepackage{amsmath}
\usepackage{amsthm}
\usepackage{amssymb}
\usepackage{wasysym}
\usepackage{fdsymbol}
\usepackage{bbm}
\usepackage{verbatim}
\usepackage{graphicx}
\usepackage{afterpage}
\usepackage{etoolbox}
\usepackage{float}
\usepackage{algorithm}
\usepackage[noend]{algpseudocode}
\usepackage{multirow}
\usepackage{rotating}
\usepackage[inline]{enumitem}
\usepackage{xcolor}
\usepackage{cite}

\usepackage[caption=false,font=footnotesize]{subfig}

\usepackage{stfloats}
\usepackage{url}
\DeclareRobustCommand*{\IEEEauthorrefmark}[1]{%
  \raisebox{0pt}[0pt][0pt]{\textsuperscript{\footnotesize #1}}%
}

\numberwithin{equation}{section}
\numberwithin{figure}{section}
\numberwithin{table}{section}

\makeatletter
\renewcommand{\p@subfigure}{\thefigure}
\makeatother

\newtheorem{definition}{Definition}%[section]
\newtheorem{theorem}{Theorem}%[section]
\newtheorem{proposition}[theorem]{Proposition}
\newtheorem{lemma}[theorem]{Lemma}

\newcommand{\repeatable}[2]{\makeatletter \global\expandafter\def\csname repText@#1\endcsname {#2} \makeatother #2}
\newcommand{\repeatxt}[1]{\makeatletter \expandafter\csname repText@#1\endcsname \makeatother}

\newtoggle{citeInAbstract}
\toggletrue{citeInAbstract}

\newcommand{\usecrop}[2]
{
	\newlength{\cropwidth}
	\setlength{\cropwidth}{\the\textwidth}
	\addtolength{\cropwidth}{#1}
	\newlength{\cropheight}
	\setlength{\cropheight}{\the\textheight}
	\addtolength{\cropheight}{#2}
	\usepackage[width=\the\cropwidth,height=\the\cropheight,center]{crop}
}

\DeclareMathAlphabet{\mathpzc}{OT1}{pzc}{m}{it}

\DeclareMathOperator*{\argmax}{\arg\,\max}
\DeclareMathOperator*{\argmin}{\arg\,\min}

\DeclareMathOperator{\vol}{vol}

\newcommand{\abs}[1]{\left | #1 \right |}

\newcommand{\norm}[1]{\left \| #1 \right \|}

\usepackage{hyperref}
\newcommand{\fakefigure}[1]% #1 = label name
{\refstepcounter{figure}\label{#1}}
\newcommand{\fakeequation}[1]% #1 = label name
{\refstepcounter{equation}\label{#1}}
\newcommand{\fakealgorithm}[1]% #1 = label name
{\refstepcounter{algorithm}\label{#1}}
\newcommand{\fakeappendix}[1]% #1 = label name
{\refstepcounter{section}\label{#1}}
\newcommand{\fakesection}[1]% #1 = label name
{\refstepcounter{section}\label{#1}}
\newcommand{\fakesubsection}[1]% #1 = label name
{\refstepcounter{subsection}\label{#1}}
\newcommand{\fakesubfig}[1]% #1 = label name
{\refstepcounter{subfigure}\label{#1}}

\newcommand{\citeproofsp}[1]
{(proof of this proposition is found in Appendix~\ref{#1})}

\newcommand{\citeproofst}[1]
{Proof of this proposition can be seen in Appendix~\ref{#1}}

\begin{document}

\title{Compressed Diffusion}

\author{\IEEEauthorblockN{Scott Gigante\,\IEEEauthorrefmark{1}\,$^\dag$,
Jay S. Stanley III\,\IEEEauthorrefmark{1}\,$^\dag$,\thanks{$^\dag$ These authors contributed equally. $^{\dag\dag}$ These authors contributed equally.}
Ngan Vu\IEEEauthorrefmark{2}, 
David van Dijk\IEEEauthorrefmark{3}, \\
Kevin R. Moon\IEEEauthorrefmark{4},
Guy Wolf\,\IEEEauthorrefmark{5}\,$^{\dag\dag}$\,$^{\S\S}$, and
Smita Krishnaswamy\,\IEEEauthorrefmark{2}\,\IEEEauthorrefmark{3}\,$^{\dag\dag}$ $^{\S}$\, \thanks{$^\S$ Corresponding author. \texttt{smita.krishnaswamy@yale.edu} \newline \indent\indent~ 333 Cedar St. New Haven CT 06510 USA}
\thanks{$^{\S\S}$ Corresponding author. \texttt{guy.wolf@umontreal.ca} \newline \indent\indent~ Pav.\ Andr\'{e}-Aisenstadt, 2920 ch.d.l. Tour, Mont\'{e}al, QC H3T 1J4, Canada}
} 
\IEEEauthorblockA{\IEEEauthorrefmark{1} Computational Biology and Bioinformatics Program, \IEEEauthorrefmark{2} Department of Computer Science, \\
\IEEEauthorrefmark{3} Department of Genetics, Yale University, New Haven, CT, USA}
\IEEEauthorblockA{\IEEEauthorrefmark{4} Department of Mathematics and Statistics, Utah State University, Logan, UT, USA}
\IEEEauthorblockA{\IEEEauthorrefmark{5} Department of Mathematics and Statistics, Universit\'{e} de Montr\'{e}al, Mont\'{e}al, QC, Canada}
}

\date{}
\maketitle
% As a general rule, do not put math, special symbols or citations
% in the abstract
\begin{abstract}
Diffusion maps are a commonly used kernel-based method for manifold learning, which can reveal intrinsic structures in data and embed them in low dimensions. However, as with most kernel methods, its implementation requires a heavy computational load, reaching up to cubic complexity in the number of data points. This limits its usability in modern data analysis. Here, we present a new approach to computing the diffusion geometry, and related embeddings, from a compressed diffusion process between data regions rather than data points. Our construction is based on an adaptation of the previously proposed measure-based Gaussian correlation (MGC) kernel that robustly captures the local geometry around data points. We use this MGC kernel to efficiently compress diffusion relations from pointwise to data region resolution. Finally, a spectral embedding of the data regions provides coordinates that are used to interpolate and approximate the pointwise diffusion map embedding of data. We analyze theoretical connections between our construction and the original diffusion geometry of diffusion maps, and demonstrate the utility of our method in analyzing big datasets, where it outperforms competing approaches.
\end{abstract}

% no keywords

% For peer review papers, you can put extra information on the cover
% page as needed:
% \ifCLASSOPTIONpeerreview
% \begin{center} \bfseries EDICS Category: 3-BBND \end{center}
% \fi
%
% For peerreview papers, this IEEEtran command inserts a page break and
% creates the second title. It will be ignored for other modes.

\section{Introduction}
\label{sec:intro}

Manifold learning approaches are often used for modeling and uncovering intrinsic low dimensional structure in high dimensional data (e.g., in genomics~\cite{moon:review}). Diffusion maps~\cite{coifman:DM}, in particular, are a popular method that capture data manifolds with random walks that propagate through nonlinear pathways in the data. Transition probabilities of a Markovian diffusion process define an intrinsic diffusion distance metric that is amenable to low dimensional embedding. Indeed, by arranging transition probabilities in a row-stochastic diffusion operator, and taking its leading eigenvalues and eigenvectors, one can derive a small set of coordinates where diffusion distances are approximated as Euclidean distances, and intrinsic manifold structures are revealed.

While the embedding provided by diffusion maps is useful for data analysis, it is also challenging to apply to modern ``big data'' settings due to scalability issues. As is often the case with kernel methods, diffusion maps require the computation of pointwise transition probabilities, which is computationally quadratic in the size of the data, and an eigendecomposition which is cubic if all eigenvectors are computed. Here, we show that one can reduce the computational cost significantly by only considering transition probabilities between data regions, which can be efficiently computed from a compressed affinity kernel over a fixed-size partition of the data. Our construction is based on coarse-graining a measure-based kernel~\cite{bermanis:MGC-sampta,bermanis:MGC} with an inverse-density measure, which was recently introduced in~\cite{lindenbaum:SUGAR} for geometry-based data generation and uniform resampling of data manifolds. Here we extend the uses of this kernel to provide efficient implementation and application of diffusion maps by first embedding data regions, and then interpolating the pointwise embedding from their diffusion coordinates.

The main contributions of this work are as follows. On the theoretical side, we further establish the relations between the original diffusion framework from~\cite{coifman:DM} and the construction in~\cite{lindenbaum:SUGAR}, both at a pointwise and compressed (i.e., data-region) level. On the practical side, we suggest a novel partitioning method, the results of which indicate significant speedups in the computation of the diffusion embedding, which outperform other approaches, and enables the application of diffusion-based manifold learning well beyond the data sizes traditionally used with kernel methods.

\section{Problem setup}
\label{sec:setup}

\subsection{Preliminaries}
\label{sec:prelim}
Let $\mathcal{M} \subseteq \mathbbm{R}^m$ be a compact $d$ dimensional manifold immersed in the ambient space $\mathbbm{R}^m$, where $d \ll m$, which represents the intrinsic geometry of data sampled from it. For simplicity, the integration notation $\int \cdot \,dy$ in this paper will refer to the Lebesgue integral $\int_\mathcal{M} \cdot \, dy$ over the manifold, instead of the whole space $\mathbbm{R}^n$. Further, while (for simplicity) such integrals are written without a specific measure one can equivalently, w.l.o.g., replace $dx$ with an appropriate measure representing data sampling distribution over $\mathcal{M}$. Let $g(x,y) \triangleq \exp\left(-\Vert x - y \Vert^2 / \varepsilon\right)$, $x,y \in \mathcal{M}$, $\varepsilon > 0$, define the Gaussian kernel $Gf(x) = \int g(x,y) f(y) dy$ used in~\cite{coifman:DM} to capture local neighborhoods from data sampled from $\mathcal{M}$. Following~\cite{coifman:DM} and related work, we define the Gaussian degree $q(x) = \norm{g(x,\cdot)}_1 = \int g(x,y) dy$ and assume it provides a suitable approximation of the distribution (or local density) of data over the manifold $\mathcal{M}$. Finally, given a measure $\mu$ over the manifold, an MGC kernel~\cite{bermanis:MGC-sampta,bermanis:MGC} is defined as $k_{\mu}(x,y) = \int g(x,r) g(y,r) d\mu(r)$. Note that while we use a Gaussian kernel for the remainder of this work, the definitions and theorems to follow do not depend on the choice of $g$, so long as it is a kernel function.

The original MGC construction in~\cite{bermanis:MGC} considered measures that represent data distribution, and used the constructed MGC kernel to define diffusion maps (see Sec.~\ref{sec:DM} and~\cite{coifman:DM}) with them. Recently, it was shown in~\cite{lindenbaum:SUGAR} that an MGC kernel constructed with inverse-density measure allows separation of data geometry and distribution, which in turn allows uniform data generation over the data manifold, with applications in alleviating sampling biases in data analysis (e.g., imbalanced classification). Here, we further explore the inverse-density MGC kernel and its properties. In particular, we show this kernel enables compression of its resulting diffusion geometry into data regions, instead of data points, to efficiently capture the intrinsic manifold geometry of analyzed data.

\subsection{Diffusion Maps}
\label{sec:DM}
Diffusion maps~\cite{coifman:DM} utilize a set of local affinities to define a Markovian diffusion process over analyzed data, which captures the intrinsic data geometry via diffusion distances. The original construction~\cite{coifman:DM} defines transition probabilities between data points based on Gaussian affinities as $p(x.y) = g(x,y) / q(x)$, where it can be verified that $\int p(x,y) dy = 1$. This construction can also be generalized to other affinity kernels, such as the MGC kernel $k_{\mu}(x,y)$ from~\cite{bermanis:MGC} or variation in~\cite{lindenbaum:SUGAR}. Under mild conditions on the affinity kernel, the resulting transition probability operator has a discrete decaying spectrum of eigenvalues $1 = \lambda_0 \geq \abs{\lambda_1} \geq \abs{\lambda_2} \geq \ldots$ which are used together with their corresponding eigenfunctions $\phi_0, \phi_1, \phi_2, \ldots$ (with $\phi_0$ being constant) to achieve the diffusion map of the data. Each data point $x \in \mathcal{M}$ is embedded by this diffusion map to the diffusion coordinates $\Phi_t(x) = (\lambda_1^t \phi_1(x), \ldots, \lambda^t_\delta(x) \phi_\delta(x))$, where the exact value of $\delta$ depends on the spectrum of the transition probabilities operator $P f(x) = \int p(x,y) f(y) dy$, whose kernel is $p(x,y)$.

\section{Inverse-density MGC kernel}
\label{sec:IDMGC}

Our construction here is based on an inverse-density MGC kernel, first introduced in~\cite{lindenbaum:SUGAR}, which is defined as follows:
\begin{definition}
\label{def:IDMGC}
The inverse-density MGC (ID-MGC) kernel is defined as $k(x,y) = \int \frac{g(x,r) g(r,y)}{q(r)} dr$, $x,y \in \mathcal{M}$, with the associated integral operator $K f(x) = \int k(x,y) f(y) dy$.
\end{definition}
This kernel corresponds to $k_\mu(x,y)$ with $d\mu(x) = q^{-1}(x)dx$, where in practice $q(x)$ accounts (up to normalization) for data density over $\mathcal{M}$. Therefore, the connectivity captured by this kernel normalizes density variations by enhancing relations in sparse regions compared to dense ones. Indeed, in~\cite{lindenbaum:SUGAR} this property was used to adjust the distribution of data and provide uniform resampling of data manifolds. We note that since we consider $\mathcal{M}$ as representing the intrinsic geometry of collected (or observed) data, we expect a certain amount of data points to exist in each local region, and thus $q^{-1}(x)$ can be expected to have a finite upper bound over $\mathcal{M}$.

Several useful properties of the ID-MGC kernel and its relation to the Gaussian-based diffusion operator in Sec.~\ref{sec:DM} are summarized in the following theorem \citeproofsp{apx:proof-IDMGC}:
\begin{theorem}
\label{thm:IDMGC-props}
Let the kernel $k(\cdot,\cdot)$ and operator $K$ be defined as in Def.~\ref{def:IDMGC}. Then, this kernel (and operator) can be related to the Gaussian kernel $G$ and diffusion operator $P$ via \begin{enumerate*} \item the operator norm: $\|K\| \leq \|G\|$, \item the kernel degrees: $\|k(x,\cdot)\|_1 = \|g(x,\cdot)\|_1$, $x \in \mathcal{M}$, and \item two-step diffusion: $P^2 f(x) = \int \frac{k(x,y)}{\|k(x,\cdot)\|_1} f(y) dy$. \end{enumerate*}
\end{theorem}

\section{Compressed kernel and diffusion transitions}
\label{sec:compressed}

Let $S_1,\ldots,S_n \subset \mathcal{M}$ be a partition of the manifold into measurable subsets, such that $\mathcal{M} = \bigcup_{j=1}^n S_j$, and $S_i \cap S_j = \emptyset$ for every $i \neq j$. We define a compressed kernel over such a partition as follows:
\begin{definition}
\label{def:compressed-kernel}
Let $k(x,y)$, $x,y \in \mathcal{M}$ be defined as in Def.~\ref{def:IDMGC}. The compressed kernel over partition $\mathcal{S} = \{S_j\}_{j=1}^n$ of $\mathcal{M}$ is given by 
$$
k_\mathcal{S}(S,T) = \int_S \int_T k(x,y) dxdy, \quad S,T \in \mathcal{S},
$$
with the corresponding $n \times n$ kernel matrix given by
$$
[K_\mathcal{S}]_{ij} = k(S_i,S_j), \quad i,j = 1,\ldots,n .
$$
\end{definition}

Similar to the original diffusion map construction in Sec.~\ref{sec:DM}, we normalize the compressed kernel to get diffusion probabilities $p_\mathcal{S}(S,T) = k(S,T)/\|k(S,\cdot)\|_1$, organized in an $n \times n$ row-stochastic matrix $P_{\mathcal{S}}$. This matrix captures diffusion transition probabilities between data regions. The relation between the compressed construction and the original diffusion framework is summarized in the following theorem \citeproofsp{apx:proof-compressed}:
\begin{theorem}
\label{thm:compressed-props}
The compressed construction here can be related to the pointwise one in Sec.~\ref{sec:DM}, via \begin{enumerate*} \item operator norm (in matrix \& operator form): $\|K_{\mathcal{S}}\| \leq O(\|G\|)$ with a constant that only depends on the finite volume of $\mathcal{M}$ \item kernel degrees: $\|k_\mathcal{S}(S,\cdot)\|_1 = \int_S \| g(x,\cdot) \|_1 dx$, $S \in \mathcal{S}$, and \item diffusion probabilities: $P_\mathcal{S}(S,T) = \int \Pr[S \stackrel{1\text{ step}}{\longleadsto} r] \Pr[r \stackrel{1\text{ step}}{\longleadsto} T] \,dr$ where $\Pr[r \stackrel{1\text{ step}}{\longleadsto} T] = \int_T p(r,y) dy$ and $\Pr[S \stackrel{1\text{ step}}{\longleadsto} r] = $ $\int_S p(x,r) \Pr[x | S] dx$ with prior $\Pr[x | S] = \frac{\|g(x,\cdot)\|_1}{\|k_\mathcal{S}(S,\cdot)\|_1}$. \end{enumerate*} 
\end{theorem}

\section{Compression-based fast diffusion maps} 
\label{sec:fast-DM}

The compressed construction of $K_\mathcal{S}$ and $P_\mathcal{S}$ gives rise to a natural approximation of diffusion maps. Given a partitioning $\mathcal{S}$, we define a compressed diffusion map $\Phi^t_\mathcal{S} : \mathcal{S} \to \mathbbm{R}^\delta$ analogous to the pointwise one in~\ref{sec:DM} using eigenvectors of $P_\mathcal{S}$ (with corresponding eigenvalues) in lieu of the eigenfunctions of $P$. This diffusion map provides an embedding of data regions, rather than points, which represents a coarse version of the diffusion geometry over $\mathcal{M}$. Then, using the region-to-point transition probabilities $Pr[S \stackrel{1 step}{\longleadsto} x]$ from Theorem~\ref{thm:compressed-props}, we approximate the pointwise diffusion map as $\widetilde{\Phi}^t(x) = \sum_{j=1}^n \Phi^t_\mathcal{S}(S_j) Pr[S_j \stackrel{1 step}{\longleadsto} x]$. Results in Sec.~\ref{sec:results} indicate that with the proper choice of partitions in $\mathcal{S}$, the compressed diffusion map $\widetilde{\Phi}^t$ provides a good approximation of the original pointwise one from~\cite{coifman:DM}, while also providing significant scalability and performance advantages over both direct computation and other alternative approximations.
% (S_i) = ({\lambda_\mathcal{S}}^t_1 \phi_\mathcal{S}_1(S_i), \ldots, {\lambda_\mathcal{S}}^t_k \phi_\mathcal{S}_k(S_i))$
% retains many properties of the diffusion map of $G$, where $\lambda_\mathcal{S}$ and $\phi_\mathcal{S}$ are the eigenvalues and eigenfunctions of $P_\mathcal{S}$. 

It now remains to derive a strategy for efficiently partitioning the data manifold such that compression over $\mathcal{S}$ will suitably capture (albeit at a coarser level) the pointwise diffusion geometry defined by $P$. To this end, it is convenient to consider diffusion affinities and distances~\cite{coifman:DM,bermanis:MGC-sampta,bermanis:MGC}, rather than transition probabilities. These are defined as follows:
\begin{definition}
\label{def:diff-affinities}
The diffusion affinity~\cite{coifman:DM} kernel is the symmetric conjugate of $P$, which is defined spectrally as $A f(x) = \sum_{j=1}^\infty \lambda_j \langle \psi_j , f \rangle \psi_j(x)$ with $\psi_j = q^{1/2}(x) \phi_j(x)$.
\end{definition}

\begin{definition}
\label{def:diffusion_distance}
The diffusion distance \cite{coifman:DM} between two points $x$ and $y$ sampled from $\mathcal{M}$ is defined via the kernel function $a^t(\cdot,\cdot)$ of $A^t$, and equivalently via its eigendecomposition, as
\begin{align*}
[D^t(x,y)]^2 = D^t_2(x,y) &= \left\|a^t(x,\cdot) - a^t(y,\cdot)\right\|_2^2 \\
&= \sum_{j=0}^{\infty} \lambda_j^{2t}\left[\psi_j(x)-\psi_j(y)\right]^2
\end{align*}
\end{definition}

\begin{figure}
    \centering
    \includegraphics[width=\linewidth]{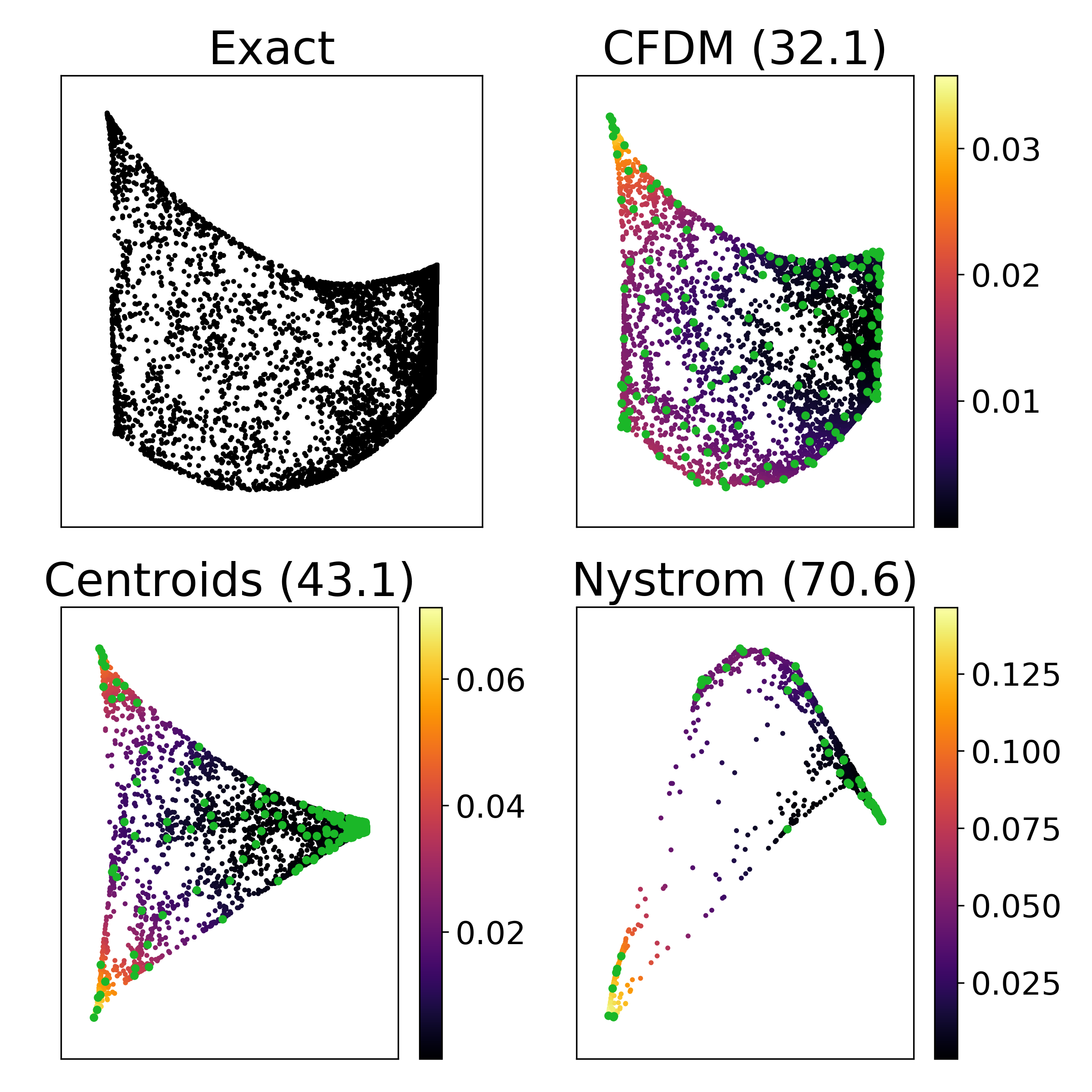}
    \caption{Example of approximations of diffusion maps on the Swiss Roll. Landmarks or centroids are shown in green. Mean squared error (MSE) of each approximation is shown in parentheses. Color shows MSE for each point.}
    \label{fig:swissroll_example}
    \vspace{-1em}
\end{figure}

We note that in the described compressed diffusion scheme, we replace pointwise affinities with affinities between data points and partitions, rather than just between partitions. Under proper partitioning of the data, we expect (or assume) such substitution would retain intrinsic structure of the data, as captured by the diffusion geometry defined by $P$ or $A$. The following proposition examines the difference between pointwise and coarse-grained diffusion affinity information, and relates this difference to locality of partitions with respect to the diffusion distance metric.
\begin{proposition}
\label{prop:partition-locality}
Let $S \in \mathcal{S}$ and $\xi > 0$. If the diffusion affinities satisfy 
$\sup\limits_{z \in \mathcal{M},x \in S_j}\left|a(x,z) - \mathbb{E}_{y \in S_j}[a(y,z)]\right| < \xi$ , then the diffusion distances $D^1(x,y)$, $x,y \in S$, are bounded from above by~$O(\xi)$, with a constant that only depends on $\vol{\mathcal{M}}$.
\end{proposition}
\begin{proof}
For any $x,y \in S_j$ and $z \in \mathcal{M}$, due to the triangle inequality, we have $|a(x,z) - a(y,z)| \leq |a(x,z) - \mathbb{E}_{u \in S_j}[a(u,z)]| + |a(y,z) - \mathbb{E}_{u \in S_j}[a(u,z)]| \leq 2\varepsilon$. Therefore, $\int |a(x,z) - a(y,z)|^2 dz \leq 4 \varepsilon^2 \vol{\mathcal{M}}$, which together with Definition~\ref{def:diffusion_distance} yields the result in the proposition.
\end{proof}
\noindent Indeed, this result indicates that to faithfully capture the diffusion geometry by given partitioning $\mathcal{S}$, the partitions should be local in the diffusion geometry, since the diffusion geometry in turn preserves the local geometry of the manifold~\cite{coifman:DM}. This understanding motivates our formulation of a partitioning strategy, as explained below.

Here we turn to practical settings of data analysis applications. Therefore, in the following we consider finite settings, particularly for some dataset $X \subset \mathcal{M}$, $|X| < \infty$, sampled from the manifold, with $P$ and $A$ being $|X| \times |X|$ matrices constructed from it. Note also that in the this setting, the eigenfunctions of $P$ or $A$ become eigenvectors (i.e., in $\mathbbm{R}^{|X|}$), and the RHS in Definition~\ref{def:diffusion_distance} is written analogously with the sum going over $|X|$ eigenpairs. 

In order to define a partition on $X$, we first choose a set of partition centroids $L$. Noting the success of coherence-based sampling strategies for signal compression and recovery\cite{puy2018random}, we propose a random sampling without replacement biased by coherence, where we estimate the coherence as follows.

\begin{proposition}
\label{prop:estimator}
Let $\mu_\ell(x) = \sum_{j=0}^{\ell} |\psi_j(x)|^2$ be the order-$\ell$ coherence of $x \in X$ as in \cite{puy2018random}. Then, the \textit{t}-step diffusion coherence $\rho_t(x) = \|a^t(x,\cdot)\|_2^2$ is an estimate of $\mu_\ell(x)$ with error 
$$
\left\|\mu_\ell(x) - \rho_t(x)\right\|_2^2 \propto \left| \ell - \sum_{j=0}^{|X|-1} \lambda_j^{2t}\right|
$$for every $x \in X$.
\end{proposition}
\noindent\citeproofst{apx:proof-estimator}.

According to Proposition~\ref{prop:partition-locality}, we should assign $x \in X$ to the partition associated with centroid $y \in L$ such that $D^t_2(x,y)$ is minimized. However, this distance is in practice biased by the diffusion coherence $\rho_t(y)$. We therefore define the \textit{angular diffusion distance}, which we will show maximizes transition probabilities between points and the assigned centroids.

\begin{definition}
\label{def:angulardiffusion}
The angular diffusion distance between $x,y \in \mathcal{M}$ is defined using the eigendecomposition of the diffusion affinity kernel from Def.~\ref{def:diff-affinities} as
$$
D^t_c(x,y) = \arccos\left(\frac{\sum_{j=0}^\infty \lambda_j^{2t}\psi_j(x)\psi_j(y)}{\sqrt{\rho_t(x)\,\rho_t(y)}} \right).
$$
\end{definition}

\begin{proposition}
\label{prop:distances}
For $x \in \mathcal{M}$, the following dualities hold:
\begin{enumerate}
\item $\displaystyle \argmin_{y\in \mathcal{M}} ~ D^t_c(x,y) =\argmax_{y\in \mathcal{M}} ~ \frac{a^{2t}(x,y)}{\sqrt{\rho_t(x)\,\rho_t(y)}}$ ;
\item $\displaystyle \argmin_{y\in \mathcal{M}} ~ \left\|\frac{a^t(x,\cdot)}{\sqrt{\rho_t(x)}} - \frac{a^t(y,\cdot)}{\sqrt{\rho_t(y)}}\right\|_2^2 =  \argmin_{y\in \mathcal{M}} ~ D_c^t(x,y)$ .
\end{enumerate}
\end{proposition}
\citeproofst{apx:distances}. Propositions \ref{prop:estimator} and \ref{prop:distances} provide a convenient method to select a set of partitions that optimize proposition \ref{prop:partition-locality}.  To do this, we assign points $x \in X$ to partitions $S_i \in \mathcal{S}$ corresponding to centroids $y_i \in L$ according to
\begin{equation*}
\label{eqn:optimal_landmarks}
S_i = \{x \in X : D^t_c(x, y_i) = \argmin_{y \in L} D^t_c(x,y)\}.
\end{equation*}

% \begin{algorithm}
% \caption{Compression-based Fast Diffusion Maps}\label{alg:mgc_diffusion}
% \begin{algorithmic}[1]
% \Procedure{FastDiffusion}{$X = \{x_1,\ldots,x_N\}, n, t$}
% \State $G \gets GaussianKernel(X)$
% \State $Q \gets \text{diag}(RowSum(G))$ 
% \State $K \gets G Q^{-1} G $; $A \gets Q^{-1/2} G Q^{-1/2}$
% \State $\rho \gets \|A^t(1,\cdot)\|_2^2, \ldots, \|A^t(n,\cdot)\|_2^2$
% \State $c_1,\ldots,c_n \gets Sample(N, k=n, p=\rho/sum(\rho))$
% \State \Comment{Assign $x$ to the nearest centroid}
% \State $S_1,\ldots,S_n \gets \emptyset$
% \For{$x \in X$}
%     \State $D \gets [0, \cdots, 0]$
%     \For{$j=1,\ldots,n$}
%         \State $D(j) \gets \rho(x) + \rho(c_j) - 2A^{2t}(x,c_j)$
%     \EndFor
%     \State Let $j = \argmin(D)$; $S_j \gets S_j \cup \{x\}$
% \EndFor
% \State \Comment{Build the compressed kernel}
% \For{$i,j = 1,\ldots,n$}
%     \State $k_\mathcal{S}(S_i, S_j) \gets 0$
%     \For{$x \in S_i$, $y \in S_j$}
%         \State $k_\mathcal{S}(S_i, S_j) \gets K_\mathcal{S}(S_i, S_j) + k(x,y)$
%     \EndFor
% \EndFor
% \State $\Phi_\mathcal{S} \gets DiffusionMap(K_\mathcal{S})$
% \State \Comment{Upsample the diffusion map}
% \For{$x$ in $G$}
%     \State $\Phi(x) \gets [0, \cdots, 0]$
%     \For{$j=1,\ldots,n$}
%         \State $\Phi(x) \gets \Phi(x) + \Pr[S_j \stackrel{1 \text{step}}{\longleadsto} x] \Phi_\mathcal{S}(S_j)$
%     \EndFor
% \EndFor
% \EndProcedure
% \end{algorithmic}
% \end{algorithm}

\section{Empirical Results}
\label{sec:results}

\subsection{Swiss roll}

Here we compare three techniques for approximation of diffusion maps: linear interpolated diffusion on cluster centroids, volume-weighted Nystrom extension\cite{long2017landmark}, and our proposed compression-based fast diffusion map (CFDM). % The former two methods are described in the appendix.

The Swiss roll is a canonical test dataset for diffusion maps, and manifold learning in general. Figure~\ref{fig:swissroll_example} shows the diffusion map of the Swiss roll and each approximation, with the partition diffusion coordinates plotted in green. To quantify the reconstruction, we calculated the sum of squared error (SSE) across 32 diffusion components of 200 Swiss rolls, allowing for sign flips and reordering of components. CFDM produces a lower approximation error and a smaller approximate kernel for less computational cost than competing methods (Fig.~\ref{fig:swissroll_evaluation}). For $n\ge10^4$, CFDM offers an approximately 10x speed-up over exact diffusion maps, and this advantage grows slightly as $n$ increases.

\begin{figure*}[!tb]
    \centering
    \includegraphics[width=0.95\linewidth]{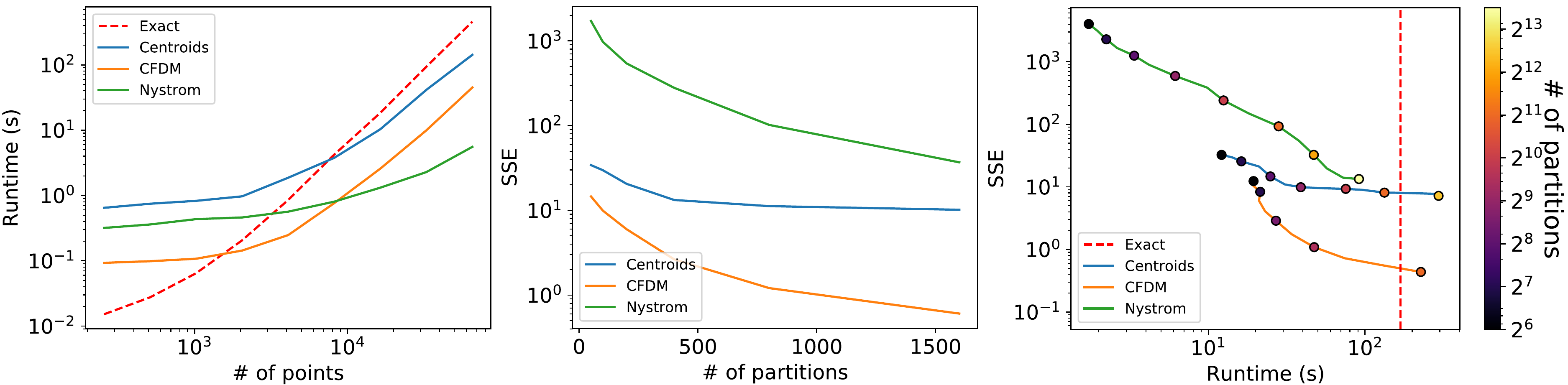}
    \caption{Quantitative evaluations. \textit{Left:} Runtime scaling for 150 partitions over an increasing number of points. \textit{Middle:} Approximation error on 32 diffusion components over an increasing number of partitions. \textit{Right:} Approximation error compared to runtime over $2^{14}$ points.}
    \label{fig:swissroll_evaluation}
    \vspace{-1em}
\end{figure*}

\subsection{Mass cytometry of induced pluripotent stem cells}

Mass cytometry is the measurement of protein abundance in individual cells via mass spectrometry. We used CFDM to approximate the diffusion map of 20,000 single cells in an induced pluripotent stem cell (iPSC) system\cite{zunder2015continuous} with 500 partitions. Figure~\ref{fig:ipsc} shows the exact and approximated diffusion maps. CFDM produces a visually equivalent embedding in a fraction of the time, accurately revealing the differentiation of skin cells into both iPSCs and a failed reprogramming state. Such an approximation will be beneficial for constructing diffusion maps on large graphs, such as social networks, as well as for downstream applications, where algorithms of high computational complexity are run on the diffusion map.

\begin{figure}[t]
    \centering
    \includegraphics[width=\linewidth]{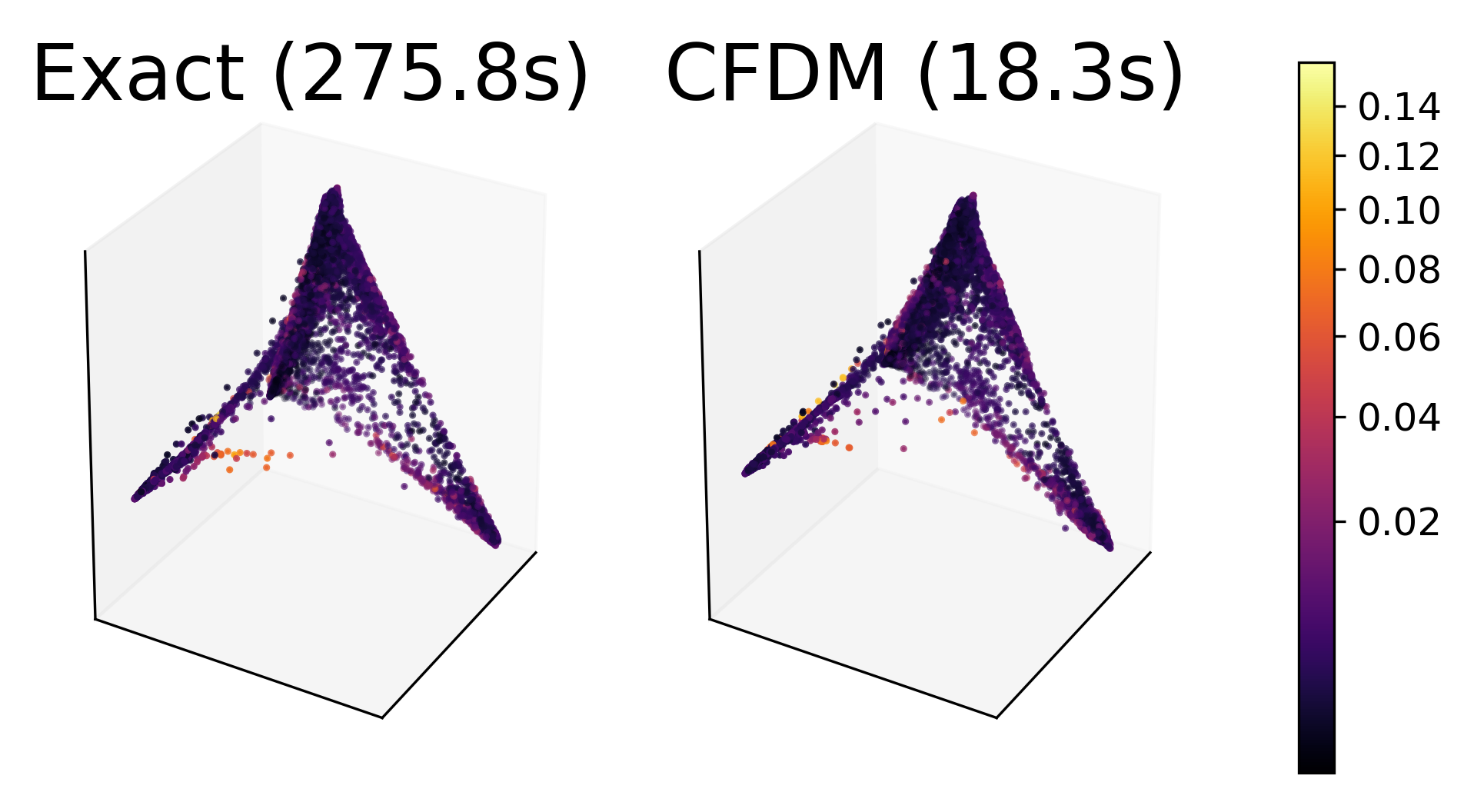}
    \caption{Exact and approximate diffusion map of iPSC dataset. Runtime for 100 components shown above. Plots are colored by SSE of approximation.}
    \label{fig:ipsc}
\end{figure}

\bibliographystyle{plain}
\bibliography{sampta,measure_kernel}

 \begin{appendices}
% \fakeappendix{apx:proof-compressed}
% \fakeappendix{apx:proof-estimator}
% \fakeequation{eq:trace-decomp}
% \fakeequation{eq:trace-orthogonality}
% \fakeappendix{apx:distances}
\onecolumn
\section{Proof of Theorem~\ref{thm:IDMGC-props}}
\label{apx:proof-IDMGC}

\noindent{}We prove the theorem in three parts, via the following lemmas. We start with the kernel degrees:

\begin{lemma}
$\Vert k(x, \cdot) \Vert_1 = \Vert g(x, \cdot) \Vert_1$
\end{lemma}

\begin{proof}
Direct computation yields $ \displaystyle
\Vert k(x, \cdot) \Vert_1 = \int k(x,y) dy = \iint \frac{g(x,r)}{\Vert g(x, \cdot) \Vert_1} g(r,y)dr dy
    = \int  \frac{g(x,r)}{\Vert g(x, \cdot) \Vert_1} \underbrace{\int g(r,y) dy}_{\Vert g(x, \cdot) \Vert_1} dr
    = \int  g(x,r) dr
    = \Vert g(x, \cdot) \Vert_1.$
\end{proof}

\noindent{}Next, we prove the relation to diffusion transition probabilities:
\begin{lemma}
$P^2 f(x) = \frac{1}{\Vert k(x, \cdot) \Vert_1} \int k(x,y) f(y) dy$
\end{lemma}

\begin{proof}
Direct computation yields $\displaystyle
P^2 f(x) = \int p(x, r) \int p(r, y) f(y) dy dr 
    = \frac{1}{\Vert g(x, \cdot) \Vert_1} \int \frac{g(x,r)}{\Vert g(r, \cdot) \Vert_1} \int g(r,y) f(y) dy dr 
    = \frac{1}{\Vert g(x, \cdot) \Vert_1} \int \underbrace{ \int \frac{g(x,r)}{\Vert g(r, \cdot) \Vert_1} g(r,y) dr}_{k(x,y)} f(y) dy
    = \frac{1}{\Vert k(x, \cdot) \Vert_1} \int k(x,y) f(y) dy.
$
\end{proof}

\noindent{}Finally, we consider the operator norm:
\begin{lemma}
$\Vert K \Vert \le \Vert G \Vert$
\end{lemma}

\begin{proof}
Let $Q f(x) = q(x) f(x)$; thus, we can verify $G = Q P$ and by combining the previous two lemmas we get
$\Vert K \Vert = \Vert Q P^2 \Vert = \Vert G P \Vert \le \Vert G \Vert \Vert P \Vert = \Vert G \Vert$, since $P$ is row-stochastic with $\Vert P \Vert = 1$.
\end{proof}

\section{Proof of Theorem~\ref{thm:compressed-props}}
\label{apx:proof-compressed}

\noindent{}We prove the theorem in three parts, via the following lemmas. We start with the operator norm:

\begin{lemma}
$\|K_{\mathcal{S}}\| \leq [\vol{\mathcal{M}}]^{3/2} \|G\|$
\end{lemma}

\begin{proof}
Let $v \in \mathbbm{R}^n$, and define the simple function $f_v = \sum_{j=1}^n [v]_j \chi_{S_j}$, where $\chi_{S_j}$ is the characteristic (i.e., indicator) function of $S_j$. Notice that since $\mathcal{M}$ is compact with finite volume, then $\|f\|_2 = \left[\sum_{j=1}^n [v]_j^2 \vol{S_j}\right]^{1/2} \leq [\vol{\mathcal{M}}]^{1/2} \|v\|_2$. Further, for $i=1,\ldots,n$,
$\displaystyle
    \left[ K_{\mathcal{S}} v \right]_i = \sum_{j=1}^n \left [\int_{S_i} \int_{S_j} k(x,y) dy dx \right ] [v]_j
    = \int_{S_i} \int k(x,y) f_v(y) dy dx = \int_{S_i} K f (x) dx $.
Therefore,
$\displaystyle
\|K_{\mathcal{S}} v\|_2^2 = \sum_{i=1}^n \left| \int_{S_i} K f (x) dx \right|^2 \leq 
\left| \int K f (x) dx \right|^2 = \|K f\|_1^2 \leq [\vol{\mathcal{M}}]^2 \| K f\|_2^2$
where the last inequality is due to H\"{o}lder inequality. Finally, by combining these results with Theorem~\ref{thm:IDMGC-props} we get
$
\displaystyle
\|K_{\mathcal{S}} v\|_2 \leq \vol{\mathcal{M}} \|K\| \|f\|_2 \leq [\vol{\mathcal{M}}]^{3/2} \|G\| \|v\|_2
$
for every $v \in \mathbbm{R}^n$, which yields the result in the lemma.
\end{proof}

\noindent{}Next, we consider the diffusion degrees:

\begin{lemma}
$\|k_\mathcal{S}(S,\cdot)\|_1 = \int_S \| g(x,\cdot) \|_1 dx$ for all $S \in \mathcal{S}$
\end{lemma}

\begin{proof}
Direct computation yields
$\displaystyle
\|k_\mathcal{S}(S,\cdot)\|_1 = \sum_j k_\mathcal{S}(S,S_j)
    = \int_S \left[\sum_j \int_{S_j} k(x,y) dy\right] dx = \int_S \int k(x,y) dy \, dx = \int_S \|k(x,\cdot)\|_1 dx
$,
and together with Theorem~\ref{thm:IDMGC-props} we get the result in the lemma.
%     = \iint \left[\int_S \frac{g(x,r)}{\Vert g(r,\cdot)\Vert_1}dx \right]\left[\int_Tg(r,y)dy \right] dr \, dT
%     &= \int \left[\int_S \frac{g(x,r)}{\Vert g(r,\cdot)\Vert_1}dx \right] \left[ \iint_T g(r,y)dy \, dT \right] dr \\
%     &= \int \int_S \frac{g(x,r)}{\Vert g(r,\cdot)\Vert_1} dx \underbrace{\int g(r,y)dy}_{\Vert g(r,\cdot)\Vert_1} dr \\
%     &= \iint_S g(x,r) dx\,  dr \\
%     &= \int_S \int g(x,r) dr\,  dx = \int_S \Vert g(x,\cdot)\Vert_1 dx \qedhere
% \end{align*}
\end{proof}

\noindent{}Finally, we prove the relation to diffusion transition probabilities:

\begin{lemma}
 $P_\mathcal{S}(S,T) = \int \Pr[S \stackrel{1\text{ step}}{\longleadsto} r] \Pr[r \stackrel{1\text{ step}}{\longleadsto} T] \, dr$
 where 
 $\Pr[r \stackrel{1\text{ step}}{\longleadsto} T] = \int_T p(r,y) dy$
 and 
 $\Pr[S \stackrel{1\text{ step}}{\longleadsto} r] = \int_S p(x,r) \Pr[x | S] dx$
 with prior 
 $\Pr[x | S] = \frac{\|g(x,\cdot)\|_1}{\|k_\mathcal{S}(S,\cdot)\|_1}$
\end{lemma}

\begin{proof}
Direct computation yields:
\begin{align*}
    P_\mathcal{S}(S,T) &= \frac{k_\mathcal{S}(S,T)}{\|k_\mathcal{S}(S,\cdot)\|_1}
    = \frac{1}{\Vert k_\mathcal{S}(S,\cdot)\Vert_1} \int_S \int_T k(x,y) dy\, dx
    = \frac{1}{\Vert k_\mathcal{S}(S,\cdot)\Vert_1} \iint_S \int_T g(x,r) \frac{g(r,y)}{\Vert g(r, \cdot) \Vert_1} dy\, dx\, dr \\
    &= \iint_S \int_T \frac{g(x,r)}{\Vert k_\mathcal{S}(S,\cdot)\Vert_1} \frac{g(r,y)}{\Vert g(r, \cdot) \Vert_1} dy\, dx\, dr
    = \iint_S \frac{g(x,r)}{\Vert k_\mathcal{S}(S,\cdot)\Vert_1} dx  \int_T \frac{g(r,y)}{\Vert g(r, \cdot) \Vert_1} dy\, dr \\
    &= \iint_S p(x,r) \frac{\|g(x,\cdot)\|_1}{\|k_\mathcal{S}(S,\cdot)\|_1} dx \int_T p(r,y)dy\, dr
    = \iint_S p(x,r) \Pr[x | S] dx \Pr[r \stackrel{1\text{ step}}{\longleadsto} T] dr \\
    &= \int \Pr[S \stackrel{1\text{ step}}{\longleadsto} r] \Pr[r \stackrel{1\text{ step}}{\longleadsto} T]\, dr \qedhere
\end{align*}
\end{proof}

% \section{Proof of Proposition~\ref{prop:partition-locality}}
% \label{apx:proof-locality}
% \textcolor{red}{TODO: I don't know how to do this one}

\section{Proof of Proposition~\ref{prop:estimator}}
\label{apx:proof-estimator}
Let $\Psi$ be a matrix whose columns are the eigenvectors $\psi_j$, $j=0,\ldots,|X|-1$, of the $A$ (in finite settings), and $\Lambda$ be a diagonal matrix that consists of the corresponding eigenvalues $1 = \lambda_0 \geq \cdots \lambda_{|X|-1}$. Since $A$ is symmetric, it yields an orthonormal set of eigenvectors, thus $\Psi$ is an orthogonal matrix and further, we can write $A^{t} = \Psi \Lambda^{t} \Psi^T$ for any $t \geq 0$. Let $I_\ell$ be an $N \times N$ diagonal matrix in which 
$$
[I_\ell]_{ii} = \begin{cases} 1 & i \leq \ell \\ 0 & i > \ell \, . \end{cases}
$$
Finally, since we are considering finite settings, we can enumerate that dataset as $X = \{x_1,\ldots,x_N\}$, and thus the order-$\ell$ coherence defined in Prop.~\ref{prop:estimator} as $\mu_\ell(x_i) = [\Psi I_\ell \Psi^T]_{ii}$ for $i=1,\ldots,N$. Similarly, we can also write $\rho_{t}(x_i) = \|a^t(x_i,\cdot)\|_2^2 = \langle a^t(x_i,\cdot), a^t(x_i,\cdot) \rangle = [A^{2t}]_{ii}$, $i=1,\ldots,N$. 

% Let \text{diag}$(\cdot)$ denote the diagonal elements of a matrix, and let $I^N_k$ be an $N \times N$ matrix in which the first $k$ columns are equal to the identity matrix $I(\cdot,k)$, and the last $N-k$ columns are zero vectors ($\mathbf{0}_{(N\times N-k)}$).

% \begin{lemma}
% The order-$\ell$ coherence is given by 
% $$\mu_k(x) = \sum^\ell_{j=0}|\psi_j(x)|^2 = \text{diag}(\Psi I^N_\ell \Psi^T)\delta_x$$
% \end{lemma}
% \begin{proof}
% \begin{align*}
%     \text{diag}\left(\Psi I^N_k \Psi^T\right) &= \text{diag}\left(\left[\psi_0 | \psi_1 | \ldots | \psi_k | \mathbf{0}_{(N\times N-k)} \right]\Psi^T\right)\\
%     &= \text{diag}\left(\left[\psi_0 | \psi_1 | \ldots | \psi_k \right]\left[\psi_0 | \psi_1 | \ldots | \psi_\ell\right]^T\right)
% \end{align*}

% As the outer product $\left[\psi_0 | \psi_1 | \ldots | \psi_\ell \right]\left[\psi_0 | \psi_1 | \ldots | \psi_\ell \right]^T$ can be written as $\sum_i^\ell \psi_i\psi_i^T$, we have by definition of the outer product that 
% \begin{align*}
%     \text{diag}\left(\left[\psi_0 | \psi_1 | \ldots | \psi_\ell \right]\left[\psi_0 | \psi_1 | \ldots | \psi_\ell \right]^T\right)\delta_x &= diag\left(\sum_i^\ell \psi_i\psi_i^T\right)\delta_x \\
%     &= \sum_i^\ell \psi_i(x)
% \end{align*}
% \end{proof}
% With this elementary lemma we have the corollary 
% \begin{corollary}
% The \textit{t}-step diffusion coherence is given by
% $$\rho_{t}(x) = \text{diag}\left(A^{2t}\right)\delta_x$$
% \end{corollary}
% \begin{proof}
% One proceeds as before, substituting the squared eigenvalue matrix in place of the truncated identity matrix. 
% \end{proof}

With this matrix formulation, we now show the stated result $\left|\mu_\ell(x) - \rho_t(x)\right| \propto \left| \ell - \sum_{j=0}^{|X|-1} \lambda_j^{2t}\right|$ from the proposition by direction computation, as: 
\begin{align}
\left|\sum_{i=1}^N \mu_\ell(x_i)-\rho_t(x_i)\right| &=\left| \sum_{i=1}^N [\Psi I_\ell \Psi^T]_{ii} - [A^{2t}]_{ii} \right| \nonumber \\
&= \left|\text{trace}\left(\Psi I^N_\ell \Psi^T - A^{2 t}\right) \right| = \left|\text{trace} \left[ \Psi \left(I^N_\ell - \Lambda^{2t}\right) \Psi^T \right] \right| \label{eq:trace-decomp} \\
\label{eq:trace-orthogonality} &= \left| \text{trace}\left(I^N_\ell - \Lambda^{2t}\right)\right| = \left| \ell - \sum_{j=0}^{|X|-1} \lambda_j^{2t} \right|
\end{align}
where (\ref{eq:trace-decomp}) is due to the eigendecomposition of $A^{2t}$ and (\ref{eq:trace-orthogonality}) is due to the orthogonality of $\Psi$. \qed
% The linearity of the trace allows us to consider the sum of the eigenvalues corresponding to $(\Psi I^N_\ell \Psi^T)$ and  $(A^{2t})$ separately. As the former is an outer product of $\ell$
% orthogonal vectors, its rank is $\ell$, with $\ell$ eigenvalues equal to 1 and $N-\ell$ eigenvalues equal to 0.  The eigenvalues $A^{2t}$ are clearly $\lambda_j^{2t}$.  Thus,

% \begin{align*}
% \left|\sum_{x\in X} \mu_\ell(x)-\rho_t(x)\right| &= \left|\sum_{j=0}^\ell( 1- \lambda_j^{2t}) -\sum_{j={\ell+1}}^{|X|-1}\lambda_j^{2t} \right|\\
% &= \left| \ell - \sum_{j=0}^{|X|-1} (1-\lambda_j)^{2t}\right|
% \end{align*}
% \qed

\section{Proof of Proposition~\ref{prop:distances}}
\label{apx:distances}
% Let $\Psi$ be a matrix whose columns are the eigenvectors $\psi_j$, $j=0,\ldots,|X|-1$, of the $A$ (in finite settings), and $\Lambda$ be a diagonal matrix that consists of the corresponding eigenvalues $1 = \lambda_0 \geq \cdots \lambda_{|X|-1}$. Since $A$ is symmetric, it yields an orthonormal set of eigenvectors, thus $\Psi$ is an orthogonal matrix and further, we can write $A^{t} = \Psi \Lambda^{t} \Psi^T$ for any $t \geq 0$.

We note that as shown in~\cite{coifman:DM}, the spectral theorem yields $\sum_{j=0}^\infty \lambda_j^{2t} \psi_j(x_i)\psi_j(x_j) = a^{2t}(x,y) = \langle a^t(x,\cdot),a^t(y,\cdot) \rangle$ for every $t \geq 0$. Furthermore, since by definition $\rho_t(x) = \|a^t(x,\cdot)\|_2^2$, we have 
$$
\frac{\sum_{j=0}^\infty \lambda_j^{2t} \psi_j(x_i)\psi_j(x_j)}{\sqrt{\rho_t(x)\rho_t(y)}} = \left\langle \frac{a^t(x,\cdot)}{\|a^t(x,\cdot)\|_2},\frac{a^t(y,\cdot)}{\|a^t(y,\cdot)\|_2} \right\rangle \in [0,1], \quad x,y \in X,
$$
as an inner product of normalized (i.e., unit-norm ) non-negative functions. Therefore, since $\arccos$ is monotonically decreasing in the interval $[0,1]$, we get the first duality. Namely, for each $x \in \mathcal{M}$, $t \geq 0$,
$$
\argmin_{y\in \mathcal{M}} ~ D^t_c(x,y) = \argmax_{y\in\mathcal{M}} ~ \frac{\sum_{j=0}^\infty \lambda_j^{2t} \psi_j(x_i)\psi_j(x_j)}{\sqrt{\rho_t(x)\rho_t(y)}} = \argmax_{y\in\mathcal{M}} ~ \frac{a^{2t}(x,y)}{\sqrt{\rho_t(x)\,\rho_t(y)}}\,.
$$

\noindent{}To show the second duality, we first write 
$$ \left\|\frac{a^t(x,\cdot)}{\sqrt{\rho_t(x)}} - \frac{a^t(y,\cdot)}{\sqrt{\rho_t(y)}}\right\|_2^2 = \frac{\|a^t(x,\cdot)\|_2^2}{\rho_t(x)} + \frac{\|a^t(y,\cdot)\|_2^2}{\rho_t(y)} - 2 \left\langle \frac{a^t(x,\cdot)}{{\sqrt{\rho_t(x)}}} \,,\, \frac{a^t(y,\cdot)}{{\sqrt{\rho_t(y)}}} \right\rangle = 2 - 2\frac{a^{2t}(x,y)}{\sqrt{\rho_t(x)\rho_t(y)}}.
$$
Therefore, together with the first duality, we now get
$$
\argmin_{y\in \mathcal{M}} ~ \left\|\frac{a^t(x,\cdot)}{\sqrt{\rho_t(x)}} - \frac{a^t(y,\cdot)}{\sqrt{\rho_t(y)}}\right\|_2^2 = \argmin_{y\in \mathcal{M}} ~2 - ~2\frac{a^{2t}(x,y)}{\sqrt{\rho_t(x)\rho_t(y)}} = \argmax_{y\in \mathcal{M}}~\frac{a^{2t}(x,y)}{\sqrt{\rho_t(x)\rho_t(y)}} = \argmin_{y\in \mathcal{M}} ~ D^t_c(x,y) ,
$$
which completes the proof. \qed

\end{appendices}

\end{document}